\documentclass[11pt]{article}

\usepackage[margin=1.05in]{geometry}
\usepackage[T1]{fontenc}
\usepackage[utf8]{inputenc}
\usepackage{lmodern}
\usepackage{microtype}

\usepackage{amsmath,amssymb,amsthm}
\usepackage{mathtools}

\usepackage{graphicx}
\usepackage{booktabs}
\usepackage[font=small,labelfont=bf]{caption}
\usepackage{placeins}

\usepackage[round,authoryear]{natbib}
\usepackage[colorlinks=true,linkcolor=blue!60!black,citecolor=blue!60!black,
            urlcolor=blue!60!black]{hyperref}
\hypersetup{
  pdftitle={Learning Suffers More Than the Policy Class Under Partial
            Observability: A Closed-Form Analysis},
  pdfauthor={Idil Gozel},
}

\theoremstyle{plain}
\newtheorem{theorem}{Theorem}
\newtheorem{proposition}{Proposition}
\newtheorem{lemma}{Lemma}
\newtheorem{corollary}{Corollary}
\newtheorem{conjecture}{Conjecture}
\theoremstyle{definition}

\theoremstyle{remark}

\newcommand{\keq}{k_{\mathrm{eq}}}
\newcommand{\kstar}{k^{*}}
\newcommand{\Jopt}{J^{*}}
\newcommand{\Jml}{J_{\mathrm{ml}}}
\newcommand{\E}{\mathbb{E}}
\newcommand{\dd}{\mathrm{d}}
\newcommand{\qt}{\tilde q}

\usepackage{xcolor}

\title{Learning Suffers More Than the Policy Class\\[2pt]
       Under Partial Observability:\\[6pt]
       \large A Closed-Form Analysis\vspace{4pt}}
\author{Idil G\"ozel\\[3pt]
  \normalsize Department of Physics and Astronomy\\
  \normalsize University College London, London, UK\\[1pt]
  \normalsize \texttt{ucapgoz@ucl.ac.uk}}
\date{}

\begin{document}
\maketitle

\begin{abstract}
When a reinforcement learning agent cannot observe the full state, we usually
blame its policies: it cannot see enough to represent a good one. We show that
in a solvable case the bigger problem lies elsewhere. Even when a good policy
is available and the agent's value function is expressive enough to describe
it exactly, learning still ends up somewhere far worse.

We study a partially observed linear-quadratic problem in which a standard
actor-critic learner can be solved in closed form. At our default setting the
best policy the agent can represent is already close to optimal, costing
$10.4\%$ more than the ideal controller that observes everything. Learning
does not find it. The algorithm instead comes to rest at a policy that is
$35\%$ worse than the best one available to it, and we can say exactly where
and why.

The cause is a bias in what the critic learns rather than a limit on what the
actor can express. Because the agent cannot attribute what it sees to the
part of the state it cannot observe, the critic misreads that unexplained
variation as sharp curvature in its own value estimates, and the actor follows
that error away from the optimum. We derive closed-form expressions for the
resulting policy, for its cost, and for the one design choice that removes the
problem, which is how far the learner looks ahead before trusting its own
value estimates. Deep reinforcement learning experiments follow these
predictions closely. Notably, giving the agent memory of past observations
does not help, while changing how far it looks ahead does.
\end{abstract}

\section{Introduction}\label{sec:intro}                     

An agent that cannot see the whole state can fail for two different reasons.
It may be unable to represent a good policy, because the actions worth taking
depend on what it cannot observe. Or it may be able to represent a good
policy and still fail to find it, because the quantities it estimates while
learning are biased by the same blindness. The first failure is a property of
the policy class. The second is a property of the learning rule.

We will call the first the \emph{policy-class gap}: how much worse the best
policy the agent can represent is than the best policy that sees everything.
We will call the second the \emph{learning gap}: how much worse the policy
the learner actually settles on is than the best one it could have
represented. These two names are used in exactly this sense throughout the
paper.

The usual reading of partial observability is that the policy-class gap is
the problem. Give the agent a richer policy class, or a way to remember the
past, and the difficulty should recede. That reading is incomplete, and the
reason is easy to state. Learning also has to \emph{estimate} something,
usually a value function, and those estimates are formed from the same
impoverished observations. If the estimates are biased, the learner will
optimize the wrong objective no matter how good its policy class is. Existing
theory has largely bounded this second effect rather than located it, so it
has been hard to tell how much of the observed shortfall it accounts for.

This paper answers that question exactly in one solvable case. An agent
observes and actuates one coordinate of a linear system. A second coordinate,
which it never observes, drives the first. Costs are quadratic, so if the
agent could see both coordinates the problem would be classical. The learner
is a standard actor-critic pair, with a critic at its temporal-difference
fixed point and an actor following the policy gradient.

Two features of this setup make the two gaps separable, and both are
deliberate. First, the policy class is fixed and its optimum can be computed,
so the policy-class gap is known exactly rather than estimated. Second, the
critic's features are the correct ones: the true value functions here are
quadratic, and the critic uses quadratic features. Nothing is lost to
function approximation. What remains is the effect of the features being
computed from an incomplete observation, and nothing else.

The answer is that the learning gap dominates. At our default setting the
best memoryless policy is already close to optimal, costing only $10.4\%$
more than the ideal controller. Nevertheless, actor-critic learning comes to
rest somewhere else entirely: at a gain fifteen times larger, whose cost is
$35\%$ above the best memoryless policy. Both numbers are exact rather than
estimated. Why learning fails even though a good policy exists, and is
representable, is the question this paper answers.

\paragraph{Why the critic goes wrong.} The mechanism is worth stating before
any mathematics, because it is simple. Consider what the agent's observed
coordinate looks like in the long run. It fluctuates, and those fluctuations
have two sources: the control the agent applies, and the unobserved
coordinate pushing it around. In the stationary regime, meaning once the loop
has settled into its long-run statistical behaviour, these two exactly
balance. On average the unobserved coordinate supplies precisely the motion
that the control removes.

Now put yourself in the critic's position. It sees persistent fluctuation in
the observation, and it must express the value of a state as a function of
that observation alone. It cannot say ``this variation comes from a variable I
cannot see,'' because it has no term for that variable. With quadratic
features, the only place left to put unexplained persistence is the curvature
of the value function. So the critic reports value that curves far too
sharply. At the optimum of the policy class it reports curvature $16.7$
against a true value of $1.11$. An actor using this critic as its baseline is
told that moving to higher gain pays, and it keeps moving until the inflated
curvature is finally balanced by control cost, long past the optimum.

The key point is that nothing here is a failure of representation. The
critic's feature class can express the true value function exactly. The bias
comes from what the features are allowed to see.

\paragraph{Contributions.} Each item below is stated as one idea. Scope and
technical qualifications are given where the result appears.

\begin{itemize}\itemsep3pt
  \item \emph{An exact mechanism} (\S\ref{sec:mechanism}). We compute the
    critic's fixed point and the actor's expected update in closed form, and
    show the resulting equilibrium exists only because time is discrete.
  \item \emph{Where learning stops} (\S\ref{sec:law}). One elementary
    formula gives the gain at which learning comes to rest, as a function of
    a single ratio of exploration to disturbance strength.
  \item \emph{What it costs} (\S\ref{sec:cost}). At that gain the deployed
    cost has a simple closed form, and it holds for essentially any
    disturbance process, with no distributional assumption.
  \item \emph{Which of these is universal} (\S\ref{sec:scope}). The cost is
    distribution-free but the location is not. Five processes with identical
    autocovariance move the resting point by a factor of $3.4$ while leaving
    its cost nearly unchanged.
  \item \emph{The design choice that fixes it} (\S\ref{sec:lambda}). How far
    the learner looks ahead before trusting its own value estimates, and
    nothing else, moves the resting point back to the optimum.
  \item \emph{Evidence from deep reinforcement learning}
    (\S\ref{sec:memory}, \S\ref{sec:experiments}). Runs follow the
    closed-form predictions, architecture does not matter, and a
    pre-registered test rules out the obvious alternative explanation.
\end{itemize}

Two things this paper does not claim. It does not prove that the stochastic
iterates converge to the fixed point; every closed-form statement concerns
the expected update, and the correspondence with the iterates is verified
numerically. It does not claim that adding memory to an architecture is
useless in general, only that in this record it changed nothing that the
lookahead parameter did not already change. Section~\ref{sec:limitations}
states both carefully, together with a reproducible behaviour that the theory
does not explain.

\section{Setting: a partially observed linear--quadratic problem}
\label{sec:model}                                          

\subsection{Environment}\label{sec:env}

We study the simplest control problem in which an agent must act on a
coordinate it observes while a coordinate it does not observe carries the
dynamics' memory. Simplicity here is the point rather than a limitation. It
is what lets both gaps be computed exactly instead of estimated, and two of
the three laws we derive turn out to hold far beyond this instance. The environment has a two-dimensional state $(v,z)$: the
agent observes and actuates only $v$, while an unobserved state $z$ drives it,
\begin{equation}\label{eq:sde}
  \dd v = (z + u)\,\dd t, \qquad
  \dd z = (-\gamma z - c\,v)\,\dd t + \sigma\,\dd W, \qquad
  \sigma^2 = 2\Theta\gamma c ,
\end{equation}
with $u$ the control input and $W$ a standard Wiener process. Marginalizing the unobserved state is what makes this a partially observed problem of the classical kind. The observed coordinate then obeys a generalized Langevin equation in which $z$ acts through a memory kernel $K(\tau) = c\,e^{-\gamma\tau}$, together with a stationary noise whose strength $\sigma^2 = 2\Theta\gamma c$ makes the stationary variance of $z$ equal to $\Theta c$. In the general $N$-component version the kernel is $K(\tau)=\sum_{i} c_i e^{-\gamma_i \tau}$ with decay rates $\{\gamma_i\}$. This projection of an unobserved degree of freedom onto a memory
kernel and a matched stationary noise is classical
\citep{zwanzig1961memory,mori1965transport}. We write
$K(0)=\sum_i c_i$ for the kernel's zero-lag intensity. Every observation-based
policy therefore faces state aliasing \citep{singh1994learning}: distinct
values of $z$ are indistinguishable from the current observation alone.

The agent interacts in discrete time with period $h$: it observes
$v_n = v(nh)$ and applies a control held constant over the interval
(zero-order hold). The cost rate is $q v^2 + r u^2$, discounted at rate
$\rho$. Throughout, the \emph{default instance} is
$c=0.3$, $\gamma=1$, $\Theta=1$, $q=r=1$, $h=0.05$, $\rho=0.05$; every claim
is stated for general parameters and verified numerically at this instance.

\paragraph{Notation.} We collect here the symbols used throughout.
$\beta \coloneqq e^{-\rho h}$ is the per-step discount factor. A prime
denotes the next sample: $v' \equiv v_{n+1}$. For a policy with gain $k$ we
write $\kappa \coloneqq kh$, $\qt \coloneqq q + rk^2$, and
$\nu \coloneqq h s_e$, where $s_e$ is the exploration variance defined below.
The second moments of the stationary closed loop are
$m_2 \coloneqq \E[v^2]$ and $x \coloneqq \E[v v']$.

\subsection{Learning algorithm}\label{sec:algo}

The behaviour policy during learning is linear with Gaussian exploration,
\begin{equation}\label{eq:policy}
  u_n = -k\,v_n + \eta_n, \qquad \eta_n \sim \mathcal N(0, s_e)\ \text{i.i.d.},
\end{equation}
and deployment sets $s_e = 0$. The learner is a standard actor-critic pair
\citep{konda2003actorcritic}. The critic is the TD(0)/LSTD fixed point \citep{sutton1988td,bradtke1996lstd} over the observable features $(1, v, v^2)$. This is the \emph{correct} approximation class here: the true value functions of the problem are quadratic, so the choice of features introduces no error. The only limitation is what the features can see. The actor
follows the score-function policy gradient with the critic as baseline
\citep{williams1992reinforce}. In the stationary closed loop all variables are jointly Gaussian. By the Isserlis moment factorization \citep{isserlis1918formula}, both the critic's fixed point and the actor's expected update therefore reduce to polynomials in $m_2$ and $x$.

\subsection{Two reference values}\label{sec:refs}

Two quantities calibrate everything that follows, and both are exact for
every instance of \eqref{eq:sde}. The optimal LQG cost $\Jopt$ is
attained by Kalman filtering and certainty-equivalent control
\citep{kalman1960new,astrom1965optimal}: it is the best any causal
controller can do. The \emph{best memoryless} cost $\Jml$, attained at gain $\kstar$, is the best achievable by any deployment policy of the form $u=-kv$. Policies of this form are called memoryless, or in control terms static output feedback \citep{fatkhullin2021optimizing}. This cost is the optimum \emph{of the policy class the learner searches}.
At the default instance,
\begin{equation}\label{eq:refs}
  \Jopt = 0.1989, \qquad \Jml = 0.2196 \ \text{at}\ \kstar = 1.649,
\end{equation}
a policy-class gap of $10.4\%$. We verified separately that this gap is not
amplified by finite horizons, initial-condition effects, or risk-sensitive
objectives. Whatever the learning dynamics do beyond it is therefore not a representation effect. The distinction between $\Jml$ and what learning
actually attains is the subject of this paper.

\subsection{An exact sampling reduction}\label{sec:reduction}

The zero-order hold makes one step of the sampled dynamics available in
closed form, with no approximation.

\begin{lemma}[Exact one-step reduction]\label{lem:reduction}
Under the policy \eqref{eq:policy}, for any $h>0$ and any realization of the
unobserved state,
\begin{equation}\label{eq:ar1}
  v_{n+1} \;=\; (1-\kappa)\,v_n \;+\; h\,\bar z_n \;+\; h\,\eta_n ,
  \qquad
  \bar z_n \coloneqq \frac{1}{h}\int_{nh}^{(n+1)h} z_t \,\dd t .
\end{equation}
\end{lemma}

\begin{proof}
Integrate $\dd v = (z+u)\dd t$ over $[nh,(n+1)h)$ with $u$ held at
$u_n = -k v_n + \eta_n$. The observed coordinate carries no direct noise, so
the integral is exact.
\end{proof}

Equation \eqref{eq:ar1} says the sampled observation is a linear recursion
driven by the interval average of the unobserved state and the exploration noise. Its second-moment consequences require no distributional assumptions at all, a fact we exploit in \S\ref{sec:cost}. Gaussianity enters only through the recursion's interaction with the critic (\S\ref{sec:scope}). The stability condition of the recursion, $\kappa \in (0,2)$, is the sampled-control stability condition. The boundary $\kappa \to 2$ is worth keeping in mind: it reappears as the edge of every law we derive.

\section{The aliased critic and its consequence for the actor}
\label{sec:mechanism}                                       

\subsection{Stationarity pins the observed mean reversion}

The first structural fact is a constraint that stationarity alone places on
what any observer of $v$ can see. Write $\Delta \coloneqq v' - v$ for the
sampled increment.

\begin{lemma}[Observed mean reversion]\label{lem:unitroot}
In the stationary closed loop, for any gain $k$ and any environment
parameters, $\E[v\Delta] = -\tfrac12 \E[\Delta^2]$. Evaluated for the model
\eqref{eq:sde}, exactly at every $h$,
\begin{equation}\label{eq:unitroot}
  m_2 - x \;=\; \frac{\nu h}{2}\,\Bigl(1-\frac{\kappa}{2}\Bigr)^{-1}.
\end{equation}
\end{lemma}

The first statement is immediate from $\E[v'^2]=\E[v^2]$ and requires no
distributional assumptions; the evaluation \eqref{eq:unitroot} is proved in
Appendix~A. Its content is easy to state. In stationarity the unobserved state resupplies, on average, exactly the motion that the control removes. The autocovariance decay of the \emph{observed} coordinate is therefore set by the exploration noise alone, and it carries almost no trace of either the control effort or the unobserved dynamics. A critic restricted to functions of $v$
must attribute the resulting persistence somewhere, and a quadratic feature
class offers only one place to put it: the curvature of the value function.

\subsection{The critic's fixed point}

Before computing anything, it is worth predicting what should happen. The
critic must assign value using functions of $v$ alone. Lemma~\ref{lem:unitroot}
says the observed coordinate looks almost like a random walk, because the
control and the unobserved state cancel on average. A near-random walk visits
large values of $v$ and stays there, which to any observer looks like a state
whose value differs sharply from the mean. With features $(1,v,v^2)$ the only
coefficient that can express such a difference is the one on $v^2$. We should
therefore expect the fitted curvature to come out far too large, and the
following proposition says by how much.

\begin{proposition}[Critic fixed point, exact at every $h$]\label{prop:critic}
The TD(0)/LSTD fixed point over $(1, v, v^2)$ has first-order coefficient
zero and curvature coefficient
\begin{equation}\label{eq:theta2}
  \theta_2 \;=\; \frac{h\,\qt\, m_2^2}{m_2^2 - \beta x^2}
  \;\xrightarrow[\;h\to0\;]{}\; \frac{\qt\, m_2}{\nu + \rho\, m_2},
  \qquad\text{versus}\qquad
  \theta_2^{\mathrm{cond}} \;\approx\; \frac{\qt}{2k+\rho},
\end{equation}
where $\theta_2^{\mathrm{cond}}$ is the curvature of the true conditional
value function $\E[V \mid v]$.
\end{proposition}

At the default instance, at $k=\kstar$, the critic reports $\theta_2 = 16.7$ against a true conditional curvature of $1.11$: a fifteen-fold inflation (Figure~\ref{fig:mechanism}b), verified against an independent solve to $1.5\times10^{-13}$. The denominator of \eqref{eq:theta2} is where Lemma~\ref{lem:unitroot} acts. The quantity $m_2^2-\beta x^2$ is small precisely because stationarity forces $x$ close to $m_2$. It is kept away from zero only by the exploration term $\nu$ and the discount
$\rho$.

The key point is that this inflation is not an approximation error. The
feature class can represent the true value function exactly. Persistence that
the critic cannot attribute to the unobserved state has nowhere else to go
but curvature in $v$.

\begin{figure}[!tb]
  \centering
  \includegraphics[width=\textwidth]{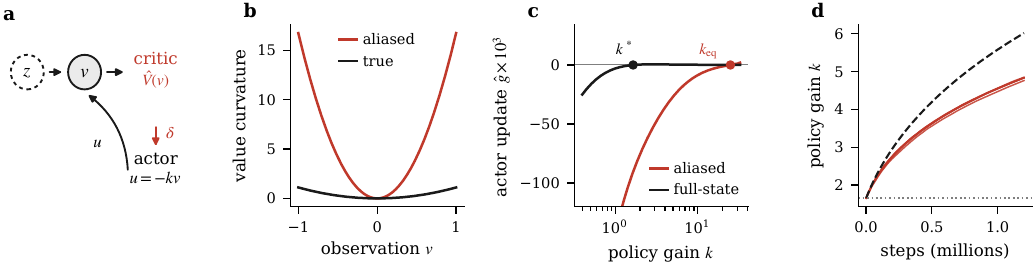}
  \caption{\textbf{An aliased critic inflates value curvature, and the actor follows
  the resulting update field away from the optimum.}
  \textbf{(a)} An unobserved state $z$ drives the observed coordinate $v$.
  The critic sees only $v$, so its temporal-difference signal $\delta$ is
  computed from an aliased state.
  \textbf{(b)} At the class optimum the critic reports curvature
  $\theta_2=16.7$ (red) against a true value of $1.11$ (black).
  \textbf{(c)} The expected update $\hat g$ (red) stays negative until
  $\keq=24.79$. A critic given the full state follows the true gradient
  (black, rescaled), which vanishes at $\kstar$.
  \textbf{(d)} Online actor-critic started at $\kstar$, six seeds (red),
  against the averaged dynamics (black dashed). All six leave the optimum,
  reaching $k=4.82\pm0.03$ at $1.2$M steps. The runs lag the averaged
  trajectory by ${\sim}20\%$ (\S\ref{sec:limitations}). Dotted line: $\kstar$.}
  \label{fig:mechanism}
\end{figure}

\subsection{The actor's expected update}

The actor performs stochastic gradient steps
$k \leftarrow k - \alpha\, \delta_n\, \partial_k \log \pi_k(u_n \mid v_n)$,
where $\delta_n$ is the temporal-difference error under the critic of
Proposition~\ref{prop:critic} and
$\partial_k \log \pi_k(u_n\mid v_n) = -\eta_n v_n / s_e$ for the policy
\eqref{eq:policy}. Let $\hat g(k)$ denote the expectation of the update
direction $\delta_n(-\eta_n v_n/s_e)$ in the stationary loop.

\begin{proposition}[Expected update]\label{prop:update}
Exactly at every $h$,
\begin{equation}\label{eq:ghat}
  \hat g(k) \;=\; 2hrk\, m_2 \;-\; 2\beta\,\theta_2\, x\, g ,
\end{equation}
where $g$ is the one-step control response
$g = \int_0^h \bigl[\partial v_t/\partial u\bigr]\,\dd t = h + O(h^3)$.
Substituting Proposition~\ref{prop:critic} and passing to the limit
$h \to 0$ at fixed $k$ (in which $x \to m_2$ and $g \to h$), the stationarity
condition $\hat g = 0$ reduces to $rk = \beta\,\theta_2(k)$, which has
\emph{no solution for any} $k>0$: in continuous time the expected update
never changes sign.
\end{proposition}

Two consequences follow. First, the coupled critic-actor dynamics do have a fixed point at any positive sampling period. It exists only through time discreteness, and it is far from the optimum of the policy class: at
the default instance $\hat g$ vanishes at $\keq = 24.79$, against
$\kstar = 1.649$ (Figure~\ref{fig:mechanism}c). A critic given the full
state instead follows the true gradient $\dd J/\dd k$ and is stationary at
$\kstar$: the displacement is an effect of aliasing, not of discounting or
of the gradient estimator. Second, \eqref{eq:ghat} is the expectation of the \emph{actual} stochastic update, scale included, not merely proportional to it. With the critic held at its fixed point, the sample average of $\delta_n(-\eta_n v_n/s_e)$ over $1.5\times10^5$ steps agrees with $\hat g$ to within $3\%$ at $k \in \{1.65, 3, 6\}$, the residual being Monte Carlo error at that sample size (Appendix~\ref{app:experiments}).

\subsection{The stochastic iterates leave the optimum}

Proposition~\ref{prop:update} concerns the averaged dynamics. To confirm
that the stochastic algorithm follows them, we initialize online TD(0)
actor-critic \emph{at the class optimum} $\kstar$ and run six independent
seeds ($\alpha_{\mathrm c}=5\times10^{-3}$,
$\alpha_{\mathrm a}=10^{-4}$, critic warm-started at its fixed point). All
six leave the optimum monotonically, reaching $k = 4.82 \pm 0.03$ after
$1.2$M steps (Figure~\ref{fig:mechanism}d). The trajectory of the averaged dynamics, obtained by integrating \eqref{eq:ghat} at the same step size, reaches $6.02$. The iterates track its shape but lag it by roughly $20\%$,
because the online critic trails its instantaneous fixed point as $k$
moves. We return to this gap, the difference between the averaged dynamics and the stochastic iterates, in \S\ref{sec:limitations}. Nothing in this paper claims the iterates converge to $\keq$. The claims are that the expected update vanishes there and that the iterates leave $\kstar$ along its field.

\section{Fixed points of the expected update: a closed-form law}
\label{sec:law}                                             

We now know why the learner moves away from the optimum. The next question
is where it stops.

Proposition~\ref{prop:update} places the fixed point of the expected update
at gains that grow as the sampling period shrinks. The natural regime in
which to study it is therefore the joint limit $h \to 0$ with $\kappa = kh$
held fixed. This is high-gain feedback, with the gain measured in units of the sampling rate. In this limit the second moments obey
$m_2 = h^2 M(\kappa)\,(1+o(1))$ and $x = h^2 X(\kappa)\,(1+o(1))$ with
\begin{equation}\label{eq:MX}
  M(\kappa) = \frac{\Theta K(0)}{\kappa^{2}}
            + \frac{s_e}{\kappa(2-\kappa)},
  \qquad
  X(\kappa) = (1-\kappa)\,M(\kappa) + \frac{\Theta K(0)}{\kappa},
\end{equation}
and the fixed-point condition $\hat g = 0$ reduces to
$M^2 - X^2 = \kappa X M$ (Appendix~A). Everything about the environment now
enters through a single ratio,
\begin{equation}
  \sigma \;\coloneqq\; \frac{s_e}{\Theta K(0)} ,
\end{equation}
the exploration power measured against the zero-lag intensity of the memory
kernel.

\begin{theorem}[Closed form for the fixed point]\label{thm:chi}
For every $\sigma > 0$ the condition $M^2 - X^2 = \kappa X M$ has exactly one root $\chi \in (0,2)$. The root is given in closed form by the inverse of
\begin{equation}\label{eq:chiclosed}
  \sigma(\chi) \;=\; \frac{(2-\chi)\bigl(\sqrt{4+\chi^{2}}-\chi\bigr)}{2\chi},
  \qquad\text{equivalently}\qquad
  \sigma \chi^{2}\,(\sigma + 2 - \chi) \;=\; (2-\chi)^{2} .
\end{equation}
The root is simple: both factors of $\sigma(\chi)$ are positive and strictly
decreasing on $(0,2)$, so $\sigma(\cdot)$ is a strictly decreasing bijection
of $(0,2)$ onto $(0,\infty)$.
\end{theorem}

The monotonicity argument is the entire uniqueness proof: for the second
factor, $\frac{\dd}{\dd\chi}\bigl(\sqrt{4+\chi^2}-\chi\bigr)
= \chi/\sqrt{4+\chi^2}-1 < 0$. The algebraic reduction from
$M^2-X^2=\kappa XM$ to \eqref{eq:chiclosed} is elementary and given in
Appendix~A; we verified the two expressions agree with a direct numerical
root-find to $2\times10^{-14}$ across six decades of $\sigma$.

Three consequences are immediate. First, the fixed-point gain is pinned to
the sampling rate, $\keq = \chi(\sigma)/h$: halving the sampling period
doubles the gain at which the expected update vanishes. Second, the limits
of \eqref{eq:chiclosed} are explicit,
\begin{equation}\label{eq:asymptotics}
  \chi = 2 - \bigl(2\sqrt{2}+2\bigr)\,\sigma + O(\sigma^{2})
  \quad (\sigma \to 0),
  \qquad
  \chi = \frac{2}{\sigma} + O(\sigma^{-2})
  \quad (\sigma \to \infty),
\end{equation}
so as exploration vanishes the fixed point approaches the sampled-control
stability boundary $\kappa = 2$, while for strong exploration
$\keq \approx 2\Theta K(0)/(s_e h)$: the fixed-point gain is inversely
proportional to the exploration power. The latter is a quantitative version
of the practitioner observation that larger exploration noise mitigates the
displacement. Third, the kernel enters \eqref{eq:chiclosed} only through
$K(0)$: two kernels of equal zero-lag intensity but different shape give the
same $\chi$, even though they give different optimal LQG and best-memoryless
costs. Figure~\ref{fig:laws}c shows twelve instances spanning $h$, $\gamma$,
exploration, kernel shape and one to three unobserved components falling
onto the curve \eqref{eq:chiclosed}; the curve is the closed form, not a
fit. Figure~\ref{fig:laws}a shows the resulting dissociation: as the
unobserved state is made faster, the policy-class gap stays nearly flat
while the learning gap grows. The two gaps are governed by different
properties of the same environment.

\begin{figure}[!tb]
  \centering
  \includegraphics[width=\textwidth]{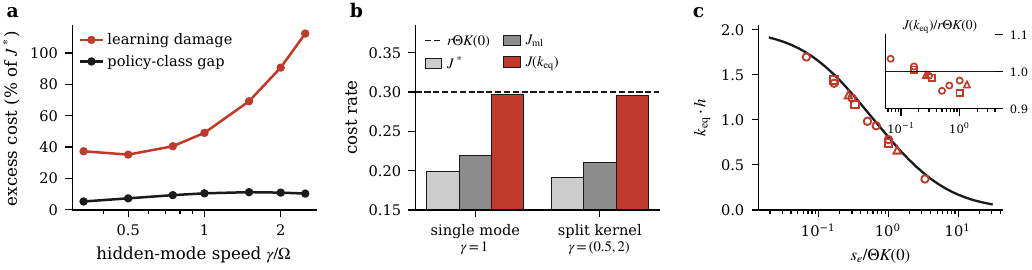}
  \caption{\textbf{The two gaps, and the law that governs the learning gap.}
  \textbf{(a)} Excess cost over $\Jopt$ as the unobserved state is made
  faster. The policy-class gap (black) is single-digit and nearly flat. The
  learning gap (red) is several times larger and \emph{grows} as memory
  becomes less important for the policy class.
  \textbf{(b)} Kernel intensity, not shape, sets the cost at the fixed point.
  Two kernels with the same $K(0)=0.3$ give different optimal and
  best-memoryless costs but the same cost at the fixed point, at the ceiling
  $r\Theta K(0)=0.300$ (dashed).
  \textbf{(c)} The fixed-point law in the joint limit. Markers: twelve
  instances spanning $h$, $\gamma$, exploration, kernel shape and component
  count (circles, squares, triangles for $N=1,2,3$), against the closed form
  \eqref{eq:chiclosed}, a curve with no fitted parameters. Inset: the same
  instances' cost against the ceiling.}
  \label{fig:laws}
\end{figure}

\paragraph{Rate of convergence.}
How fast is the limit attained? Consider the variant of \eqref{eq:sde} in which the feedback of $v$ on $z$ is removed, so that the unobserved state becomes an exogenous stationary disturbance and Lemma~\ref{lem:reduction} closes the loop exactly. For this variant the convergence is first order with an explicit coefficient.

\begin{theorem}[Finite sampling period]\label{thm:rate}
For the exogenous-disturbance variant, with the disturbance the stationary Gaussian process with autocovariance $\Theta K(\tau)$, $K(\tau)=\sum_i c_i e^{-\gamma_i\tau}$,
\begin{equation}
  \keq\, h \;=\; \chi(\sigma) \;+\; c_1 h \;+\; O(h^{2}),
\end{equation}
where $c_1$ depends on the kernel only through the weighted mean of the decay
rates, $\bar\gamma = \sum_i c_i\gamma_i / K(0)$, and does so affinely. All constants are explicit
(Appendix~C).
\end{theorem}

Numerically, at $\sigma=0.3$, $\rho=0.05$: $c_1 = -0.891$, $-1.323$ and
$-1.539$ for $\gamma = 1$, $\gamma = 1.5$ and an equal mixture of
$\gamma \in \{0.5, 3\}$ respectively; the affine dependence predicts the
mixture coefficient from the two pure kernels to four digits. For the full
model \eqref{eq:sde} the feedback of $v$ on $z$ contributes a further
correction of the same order, quantified in Appendix~C. At the default instance it moves $\keq$ by $-0.4\%$, an order of magnitude less than the
kernel-averaging term. At $h = 0.05$ the exact fixed point of the full
model sits at $\keq = 24.79$ against the limit value
$\chi(0.3)/h = 25.80$.

\begin{corollary}[Multiple unobserved components]\label{cor:multimode}
For any finite kernel $K(\tau)=\sum_i c_i e^{-\gamma_i \tau}$ with Gaussian
stationary components, Theorems~\ref{thm:chi} and~\ref{thm:rate} hold with
$K(0)=\sum_i c_i$; no other property of the kernel enters the limit.
\end{corollary}

The Gaussian qualifier in Corollary~\ref{cor:multimode} is not a
technicality. Section~\ref{sec:scope} shows that the fixed-point location is
\emph{not} invariant to the disturbance's distribution at fixed autocovariance; only the cost of \S\ref{sec:cost} is. The qualifier marks a real boundary of the law, not a limitation of the proof technique.

\section{The deployed cost under high-gain feedback}\label{sec:cost} 

Knowing where the learner stops is only half of what we want. The other half
is what stopping there costs.

This section prices the fixed point. The result is the most general in the
paper. It requires neither Gaussianity nor any parametric structure of the
unobserved dynamics, and it says what a controller at high gain is actually
doing with its control effort.

Consider deployment ($s_e=0$) of a static output-feedback policy at gain
$k=\kappa/h$ on the observed coordinate
\begin{equation}\label{eq:general-drive}
  \frac{\dd v}{\dd t} \;=\; \zeta_t + u ,
\end{equation}
where the disturbance $\zeta$ is \emph{any} stationary, ergodic,
mean-square-continuous process with autocovariance
$C(\tau)=\E[\zeta_0\zeta_\tau]$ and $C(0)<\infty$. The only discrete-time element is the zero-order hold on $u$. The model of
\S\ref{sec:model} corresponds to the exogenous-disturbance variant with
$C(\tau)=\Theta K(\tau)$; here nothing else about $\zeta$ is assumed.

\begin{theorem}[Deployed cost]\label{thm:cost}
For every fixed $\kappa\in(0,2)$, in the stationary regime of
\eqref{eq:general-drive},
\begin{equation}
  J(\kappa/h;\,h)\;\xrightarrow[\;h\to0\;]{}\; r\,C(0),
  \qquad
  q\,\E[v^2]\to 0,
  \qquad
  \E\bigl[(u_n+\bar\zeta_n)^2\bigr]\to 0 ,
\end{equation}
where $\bar\zeta_n$ is the interval average of the disturbance.
\end{theorem}

The third limit is the interpretive content. At high gain the sampled controller performs disturbance rejection: $u$ converges in $L^2$ to the negative of the disturbance. The deployed cost is exactly the power of the signal being cancelled, at price $r$. The state is pinned
essentially for free; the entire cost is the cancelling control effort.

The proof (Appendix~B) is short enough to sketch. At deployment
Lemma~\ref{lem:reduction} gives the exact recursion
$v_{n+1}=(1-\kappa)v_n + h\bar\zeta_n$, whose stationary solution is a
linear filter of $\{\bar\zeta_n\}$ with geometrically decaying weights.
Its second moment is a lattice sum of interval-average autocovariances $\bar C(m)$. Cauchy-Schwarz gives $|\bar C(m)|\le C(0)$, and mean-square continuity gives $\bar C(m)\to C(0)$ for each fixed lag as $h\to0$. Dominated convergence then yields $m_2 = h^2C(0)/\kappa^2\,(1+o(1))$. Then
$\E[u^2]=(\kappa/h)^2m_2\to C(0)$ while $q\,\E[v^2]=O(h^2)$. No
distributional property of $\zeta$ beyond its autocovariance ever enters:
the second-moment structure of a linear recursion is blind to everything
else.

\begin{corollary}[Cost at the learning fixed point]\label{cor:cost-at-fp}
Wherever the learning dynamics place their fixed point at a gain of order $1/h$, that is $\keq h \to \chi \in (0,2)$ as established in \S\ref{sec:law}, the deployed cost at the fixed point satisfies
$J(\keq) = r\,C(0)\,(1+O(h))$, \emph{regardless of the distribution of the
disturbance and regardless of where in $(0,2)$ the value $\chi$ falls}.
\end{corollary}

The corollary decouples two questions that the next section shows are
genuinely distinct: \emph{where} the fixed point sits (which does depend on
the distribution) and \emph{what it costs} (which does not).

For exponentially correlated disturbances, $C(\tau)=C(0)e^{-\gamma\tau}$,
the finite-$h$ cost has a closed form (Appendix~B), with leading correction
\begin{equation}\label{eq:Jcorrection}
  \frac{J}{r\,C(0)} \;=\; 1 \;-\;
  \frac{2\gamma h\,(1-\kappa)}{\kappa(2-\kappa)} \;-\; \frac{\gamma h}{3}
  \;+\; O\!\bigl((\gamma h)^2\bigr),
\end{equation}
large and negative at small $\kappa$, changing sign near $\kappa=1$. Table~\ref{tab:cost}
tests the closed form against simulations of five disturbance processes
sharing the autocovariance $C(\tau)=0.3\,e^{-\tau}$ but differing in
distribution (constructed in \S\ref{sec:scope}), each evaluated at its own
learning fixed point. The formula has no adjustable parameters, and agreement is $0.5$ to $2.5\%$. The largest deviation from the limiting value $rC(0)=0.300$ is the two-state Markov (telegraph) disturbance, measured $6.7\%$ below it. The closed form at its $\kappa=0.438$ predicts $7.2\%$ below, so the measurement agrees with the theory to $0.5\%$. The truncation \eqref{eq:Jcorrection} is not accurate at this $\kappa$, where $\gamma h/\kappa = 0.11$ is not small; it gives $9.9\%$. Finite-$h$ corrections account for the deviation, but the closed form is needed, not its first-order term. The $L^2$ disturbance-rejection statement is also directly measurable. The residual $\E[(u+\bar\zeta)^2]$ equals $\E[\Delta^2]/h^2$ by Lemma~\ref{lem:reduction}. It comes out at $6$ to $11\%$ of $C(0)$ at $h=0.05$ across the same five processes, an $O(\gamma h/\kappa)$ quantity as the proof predicts.

\begin{table}[!tb]
  \centering\small
  \caption{Deployed cost at the learning fixed point for five disturbance
  processes with identical autocovariance $C(\tau)=0.3e^{-\tau}$
  (\S\ref{sec:scope}), against the closed form of Appendix~B evaluated at
  each process's own $\kappa=\keq h$. No fitted parameters. The residual for
  the two heavy-tailed processes is consistent with slower Monte Carlo
  averaging at equal sample size (Appendix~E).}
  \label{tab:cost}
  \begin{tabular}{lccccc}
    \toprule
    disturbance & $\kappa=\keq h$ & $J$ (closed form) & $J$ (simulated) & error \\
    \midrule
    two-state Markov         & 0.438 & 0.2785 & 0.2798 & $-0.5\%$ \\
    Ornstein-Uhlenbeck      & 1.231 & 0.2997 & 0.2985 & $+0.4\%$ \\
    difference of squares    & 1.456 & 0.3041 & 0.3015 & $+0.9\%$ \\
    squared OU (centred)     & 1.494 & 0.3051 & 0.2975 & $+2.5\%$ \\
    squared OU, symmetrised  & 1.494 & 0.3051 & 0.2981 & $+2.3\%$ \\
    \bottomrule
  \end{tabular}
\end{table}

\section{What the kernel controls and what the distribution controls}
\label{sec:scope}                                           

We now have two laws, one for where the learner stops and one for what that
costs. They look equally general. They are not, and the difference matters.

Theorem~\ref{thm:chi} says the fixed point depends on the environment only
through $K(0)$. That statement was proved for Gaussian unobserved dynamics,
and any test run inside the model stays within that family. There is a
structural reason why more evidence of the same kind cannot dismiss the
qualifier.

The critic's normal equations over $(1,v,v^2)$ involve fourth moments of the
stationary loop, $\mathrm{Var}(v^2)$ and $\mathrm{Cov}(v^2,v'^2)$. The closed
form of Proposition~\ref{prop:critic} rests on their Isserlis reduction to
$m_2$ and $x$. This is the only place Gaussianity enters the analysis. At
high gain the loop keeps $v$ close to a filtered copy of the disturbance,
so what feeds the curvature coefficient is the second-order structure of
the \emph{squared} disturbance $\zeta^2$. For any Gaussian process,
however, $\mathrm{Cov}(\zeta_0^2,\zeta_\tau^2) = 2C(\tau)^2$: the
statistics of $\zeta^2$ are completely determined by the autocovariance.
Within a Gaussian family, $C(0)$ and $\mathrm{Var}(\zeta^2)$ cannot be varied independently. No amount of in-family evidence, including the twelve instances of Figure~\ref{fig:laws}c, could therefore distinguish ``the fixed
point depends on $C(0)$'' from ``the fixed point depends on
$\mathrm{Var}(\zeta^2)$''. Separating the two requires leaving the Gaussian
family while holding the autocovariance fixed.

\subsection{Five disturbance processes, one autocovariance}

We construct five stationary processes, all with
$C(\tau)=0.3\,e^{-\tau}$ exactly, spanning $\mathrm{Var}(\zeta^2)$ from
zero to $1.26$ and including a matched pair that isolates skewness:

\begin{enumerate}\itemsep2pt
  \item \emph{two-state Markov (telegraph):} $\zeta \in \{\pm\sqrt{0.3}\}$
    with switching rate $1/2$. Here $\zeta^2$ is constant:
    $\mathrm{Var}(\zeta^2)=0$.
  \item \emph{Ornstein-Uhlenbeck:} the Gaussian reference;
    $\mathrm{Var}(\zeta^2)=2C(0)^2=0.18$.
  \item \emph{difference of squares:}
    $\zeta=\sqrt{0.075}\,(w_1^2-w_2^2)$ with $w_1,w_2$ independent
    unit-variance OU processes of decay rate $1/2$. Symmetric, excess
    kurtosis $6$, $\mathrm{Var}(\zeta^2)=0.72$.
  \item \emph{squared OU (centred):} $\zeta=\sqrt{0.15}\,(w^2-1)$, $w$ a
    unit-variance OU process of decay rate $1/2$. Skewness $\sqrt8$,
    $\mathrm{Var}(\zeta^2)=1.26$.
  \item \emph{squared OU, symmetrised:} process 4 multiplied by an
    independent two-state sign process of switching rate $0.01$, which
    removes the skew while changing $C(\tau)$ by at most $2\%$ over lags up
    to the disturbance correlation time. Same
    $\mathrm{Var}(\zeta^2)=1.26$ as process 4.
\end{enumerate}

Processes 4 and 5 are the matched pair: identical autocovariance and
identical $\mathrm{Var}(\zeta^2)$, differing only in odd moments. The
qualitative outcomes of runs 3 and 5 were recorded in advance of running
them (Appendix~E). For each process we locate the learning fixed point
empirically, as the zero of the sample expected update with the critic fit
by least squares on the simulated trajectory. The pipeline reproduces the exact Gaussian value to $0.7\%$ (Appendix~E).

\subsection{The fixed point moves; the cost does not}

\begin{figure}[!tb]
  \centering
  \includegraphics[width=0.72\textwidth]{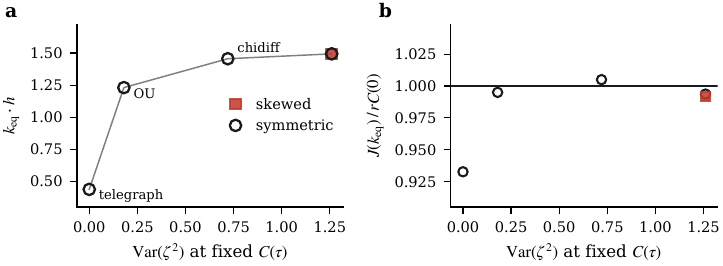}
  \caption{\textbf{The cost is distribution-free; the location is not.}
  Five disturbance processes with identical autocovariance
  $C(\tau)=0.3\,e^{-\tau}$ (\S\ref{sec:scope}).
  \textbf{(a)} The fixed-point gain moves by a factor of $3.4$, ordered by
  $\mathrm{Var}(\zeta^2)$ and saturating above the Gaussian value. Odd
  moments are irrelevant: the skewed process (red square) and its
  symmetrised counterpart (open circle) coincide. Grey line: guide to the
  eye.
  \textbf{(b)} Cost at each process's own fixed point against the limit
  $rC(0)$ (line). Four sit within $0.5\%$. The telegraph point sits $6.7\%$
  below, against $7.2\%$ predicted by the finite-$h$ closed form of
  Appendix~B at its $\kappa=0.438$.}
  \label{fig:baths}
\end{figure}

Figure~\ref{fig:baths} shows the result. The fixed-point gain spans $\keq = 8.75$ (telegraph) to $29.88$ (both squared-OU processes), a factor of $3.4$ at \emph{identical} autocovariance. It is ordered by $\mathrm{Var}(\zeta^2)$, steeply on the sub-Gaussian side and saturating above the Gaussian value. The matched pair settles the role of odd
moments: processes 4 and 5 give the same $\keq$ to within the bisection
resolution of $\pm0.4$. The mechanism reads directly off the telegraph case. With $\zeta^2$ constant, the fluctuations of $v^2$ carry far less of the disturbance's persistence, and the fitted curvature coefficient $\theta_2$ at fixed $k$ is roughly one third of its Gaussian value. What the critic misattributes is not the persistence of the observation alone but the persistence of its \emph{square}. Meanwhile the deployed cost at each
process's own fixed point stays within $0.5$ to $6.7\%$ of $rC(0)$
(Table~\ref{tab:cost}), with every deviation accounted for by
\eqref{eq:Jcorrection}: exactly the split that
Corollary~\ref{cor:cost-at-fp} predicts.

On the Gaussian side, the law extends beyond the finite-component kernels of Corollary~\ref{cor:multimode}. A Gaussian disturbance with the power-law autocovariance $C(\tau)=0.3\,(1+\tau)^{-3/2}$, which lies outside every finite-component family, gives $\keq=24.4$ at $h=0.05$. This value is inside the finite-$h$ band that Theorem~\ref{thm:rate} predicts around the limit value (Appendix~C).

\subsection{Scope of the laws}

The results of \S\ref{sec:law} and \S\ref{sec:cost} therefore split cleanly by
what each quantity depends on:

\begin{itemize}\itemsep2pt
  \item \emph{Deployed cost at the fixed point:} $r\,C(0)$, for any
    stationary ergodic mean-square-continuous disturbance
    (Theorem~\ref{thm:cost}). Proven, distribution-free.
  \item \emph{Fixed-point location:} $\keq h = \chi(s_e/C(0))$ for
    Gaussian disturbances (Theorem~\ref{thm:chi},
    Corollary~\ref{cor:multimode}); \emph{not} distribution-invariant in general. The five-process experiment is a counterexample, not an outlier.
\end{itemize}

\begin{conjecture}\label{conj:nongaussian}
At fixed autocovariance, the fixed-point gain is determined jointly by
$C(0)$ and the second-order structure of the squared disturbance
$\zeta^2$, monotonically and with saturation in
$\mathrm{Var}(\zeta^2)$, and independently of odd moments.
\end{conjecture}

The five processes above are consistent with
Conjecture~\ref{conj:nongaussian} but do not pin it down: across them the
correlation time of $\zeta^2$ varies together with
$\mathrm{Var}(\zeta^2)$, and a design that separates the two is left open.
The practical reading of this section is the paper's central asymmetry in
one sentence: to predict \emph{what partial observability will cost} a
bootstrapped learner, the disturbance's zero-lag power suffices; to
predict \emph{where the learner's update field will vanish}, its
distribution matters.

\section{The bootstrap horizon}\label{sec:lambda}           

So far the picture is entirely negative. This section gives the one design
choice that undoes it.

Everything so far bootstraps over a single step: the critic of
Proposition~\ref{prop:critic} is the TD(0) fixed point, and the update
\eqref{eq:ghat} feeds back its one-step temporal-difference error. Practical
implementations interpolate between this extreme and Monte Carlo evaluation
using $\lambda$-returns \citep{sutton1988td}, in policy-gradient form
generalized advantage estimation \citep{schulman2016gae}. The parameter
$\lambda$ sets the \emph{bootstrap horizon}: the $\lambda$-return integrates
observed costs over $h/(1-\beta\lambda)$ of physical time, in expectation,
before substituting the critic's estimate for the remainder. This section computes the fixed points of the resulting expected update exactly. The horizon turns out to be the variable that removes the pathology, and the removal is not gradual: it happens across a narrow interval of $\lambda$.

The critic is now the LSTD($\lambda$) fixed point over the same features
$(1,v,v^2)$ \citep{bradtke1996lstd}, $A(\lambda)\,\theta = b(\lambda)$ with
\begin{equation}\label{eq:lstdlam}
  A(\lambda) = \sum_{j\ge 0} (\beta\lambda)^j\,
     \E\bigl[\phi_{t-j}\,(\phi_t - \beta\phi_{t+1})^{\!\top}\bigr],
  \qquad
  b(\lambda) = \sum_{j\ge 0} (\beta\lambda)^j\, \E[\phi_{t-j}\, c_t],
\end{equation}
and the actor's expected update is
$\hat g_\lambda(k) = \E\bigl[(-\eta_t v_t/s_e)\sum_{j\ge 0}
(\beta\lambda)^j \delta_{t+j}\bigr]$.

\begin{proposition}[Expected update at general $\lambda$]\label{prop:lambda}
In the stationary closed loop, every lag moment in \eqref{eq:lstdlam} and in
$\hat g_\lambda$ is a Gaussian moment, and the geometric sums evaluate to
resolvents in the closed-loop transition matrix $F_k$. As a result, $\hat g_\lambda(k)$ is an explicit algebraic function of $(I-\beta\lambda F_k)^{-1}$ and $(I-\beta\lambda\, F_k\!\otimes\!F_k)^{-1}$, exact at every $h$ and every $\lambda\in[0,1]$. At $\lambda=0$ it reduces to Proposition~\ref{prop:update}. The expressions are collected in Appendix~\ref{app:lambda}.
\end{proposition}

At the default instance the root structure is as follows (dense scans in
$k$ at each $\lambda$; all statements verified numerically). The update has
a \emph{single} root at every $\lambda$: there is no fold, and no second
fixed point appears or disappears. Instead the one equilibrium migrates continuously to the policy optimum as $\lambda \to 1$, starting from the aliased location of \S\ref{sec:law} (the $\lambda=0$ slice of this section, $\keq h = \chi + O(h)$). The migration is concentrated in a narrow interval:

\begin{table}[!tb]
  \centering\small
  \begin{tabular}{lcccccccc}
    \toprule
    $\lambda$ & $0$ & $0.9$ & $0.95$ & $0.98$ & $0.984$ & $0.986$ & $0.99$ & $0.999$\\
    \midrule
    $\keq$ & $24.79$ & $20.55$ & $17.67$ & $11.40$ & $8.61$ & $3.43$ & $2.20$ & $1.55$\\
    $J(\keq)$ & $0.2964$ & $0.2921$ & $0.2888$ & $0.2781$ & $0.2699$ & $0.2365$ & $0.2229$ & $0.2198$\\
    excess over $\Jml$ & $35.0\%$ & $33.0\%$ & $31.5\%$ & $26.6\%$ & $22.9\%$ & $7.7\%$ & $1.5\%$ & $0.1\%$\\
    \bottomrule
  \end{tabular}
  \caption{The equilibrium of the expected GAE($\lambda$) update at the
  default instance. The excess cost drops from a third to one percent
  inside $\lambda \in [0.98, 0.99]$; the excess halves at
  $\lambda_{1/2} = 0.9854$.}
  \label{tab:lambda}
\end{table}

Three features of Table~\ref{tab:lambda} matter later. First, the excess is essentially flat up to $\lambda = 0.95$, a value widely considered conservative in practice, and disappears within the next four hundredths. Second, the $\lambda \to 1$ endpoint is $1.55$, near but not at
$\kstar = 1.649$: the residual is discounting bias, the mismatch between
the discounted learning objective and the average deployed cost. It scales
linearly with $\rho$ (reducing $\rho$ five-fold, from $0.05$ to $0.01$,
shrinks the offset from $6.1\%$ to $1.2\%$ of $\kstar$). Third, the location of the transition is set, to first approximation, by the horizon matching the kernel. Across $\gamma \in \{0.5, 1, 2\}$ and $h \in \{0.025, 0.05, 0.1\}$, a sixteen-fold variation in $\gamma h$, the half-excess point $\lambda_{1/2}$ sits where the horizon $h/(1-\beta\lambda)$ equals $1.9$ to $3.4$ correlation times $1/\gamma$ of the unobserved state. The matching rule is a heuristic, not an invariant;
the exact statement is the root of Proposition~\ref{prop:lambda}.

The heuristic still has predictive force, because it moves the transition
with $\gamma$: at $\lambda = 0.97$ the model puts the equilibrium excess at
$29.6\%$ for $\gamma=1$ but $3.9\%$ for $\gamma=2$. Those are opposite verdicts at the same $\lambda$. No single value of $\lambda$ is safe or unsafe by
itself; safety is a property of the horizon relative to the environment's
memory. Figure~\ref{fig:deeprl}(a) plots $J(\keq(\lambda))$ as a curve and
carries this section's law into deep reinforcement learning, where
\S\ref{sec:memory} tests it.

\begin{figure}[!tb]
  \centering
  \includegraphics[width=\textwidth]{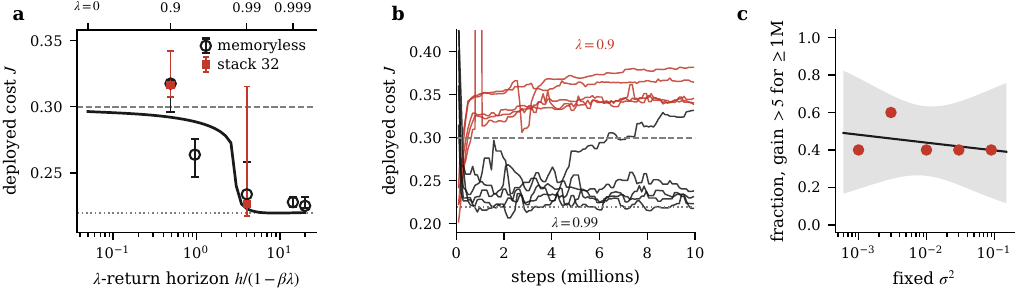}
  \caption{\textbf{The bootstrap horizon sets the outcome, and the memoryless closed
  form predicts it.}
  \textbf{(a)} Cost at the equilibrium of the exact memoryless expected
  update (curve, \S\ref{sec:lambda}) against the $\lambda$-return horizon
  (top axis: $\lambda$). Points: PPO at $10^7$ steps, median of five seeds
  (whiskers: range), memoryless (open circles) and 32-frame stack (red
  squares). The two architectures coincide, and no run improves on
  $\Jml$ (dotted). Dashed: the ceiling $r\Theta K(0)$. Deviations are
  finite-time effects (\S\ref{sec:memory}).
  \textbf{(b)} The controlled contrast at identical fixed exploration
  variance $\sigma^2=0.09$, differing only in $\lambda$ (five seeds each,
  3-sample running median). Every $\lambda=0.9$ seed (red) leaves the
  low-cost solution; four of five $\lambda=0.99$ seeds (black) hold near
  $\Jml$.
  \textbf{(c)} Fraction of seeds sustaining implied gain $>5$ for
  $\geq 10^6$ steps, across fixed $\sigma^2$ spanning two decades. Logistic
  fit with 95\% band: slope $-0.17\pm 0.58$. No threshold exists.}
  \label{fig:deeprl}
\end{figure}

\section{Does this happen in deep reinforcement learning?}
\label{sec:memory}                                          

The closed-form results concern a linear policy, a quadratic critic, and
expected updates. This section asks whether they predict the behaviour of a
standard deep implementation. We run proximal policy optimization
\citep{schulman2017ppo} on the default instance, using a public reference
implementation \citep{huang2022cleanrl} at a fixed version, with
multilayer-perceptron actor and
critic. Section~\ref{sec:experiments} describes the harness, the
pre-registration protocol, and the frozen decision rules; here we report
the outcomes.

Two architectures are compared. Memoryless agents observe $v_n$ alone.
Frame-stack agents observe the last 32 samples of $v$, a window spanning
$1.6$ correlation times of the unobserved state, so the input contains the
information needed to estimate $z$. Every run trains for $10^7$ steps, at
five seeds per configuration. We measure two quantities at deployment,
with exploration noise removed. The first is the deployed cost $J$,
estimated on one long evaluation episode. The second is the implied gain,
defined as the regression coefficient of $-u$ on $v$ under the deployed
policy.

Figure~\ref{fig:deeprl}(a) overlays the runs on the law of
\S\ref{sec:lambda}. Three facts stand out. First, the runs follow the
memoryless equilibrium curve. At $\lambda=0.99$ the memoryless median is
$J=0.2338$, against an equilibrium value of $0.2229$. At $\lambda=0.9$ the
memoryless median is $0.3175$, on the aliased side and far above $\Jml$;
the seed trajectories overshoot the equilibrium value of $0.2921$, a
finite-time effect we return to in \S\ref{sec:limitations}. The implied
gains agree as well. The model puts the $\lambda=0.9$ equilibrium at
$\keq=20.55$. The seeds that sustain the high-gain regime for millions of
steps end at implied gains with median $20.4$. Second, architecture
does not enter. The frame-stack medians are $0.3165$ at $\lambda=0.9$ and
$0.2265$ at $\lambda=0.99$, indistinguishable from the memoryless agents
at the same $\lambda$. Third, no run of either architecture at any
$\lambda$ improves on $\Jml = 0.2196$. The best median in the entire
record is $0.2251$, at memoryless $\lambda=1$. The two gaps therefore
separate empirically. The learning gap opens and closes with the bootstrap
horizon, following the closed form. The policy-class gap stays open in every
run.

The natural reading of frame stacking is that memory should remove the
aliasing and with it the pathology. The record says otherwise. Stacked
agents leave the optimum at $\lambda=0.9$ exactly as memoryless agents do.
They reach the class optimum at $\lambda=0.99$ exactly as memoryless
agents do. Input memory is not used memory: the stacked critic behaves
like the memoryless critic of the exact model at every $\lambda$ tested.
Memory in the architecture is therefore neither necessary for reaching the
class optimum nor sufficient to avoid the aliased fixed point. We note for
completeness that an earlier in-house implementation did show stacked
agents reaching the class optimum at $\lambda=0.9$ over a shorter training
horizon (Appendix~E). That effect did not survive the reference
implementation at $10^7$ steps. The benefit of frame stacking is
implementation-dependent in our record; the $\lambda$ dependence is not.

Two alternative explanations are excluded by construction. The first
attributes the failure at $\lambda=0.9$ to the collapse of the learned
exploration noise late in training. Panel (b) of
Figure~\ref{fig:deeprl} shows the controlled contrast: exploration
variance held fixed at $\sigma^2=0.09$ in both arms, with only $\lambda$
differing. All five $\lambda=0.9$ seeds leave the low-cost solution; four
of five $\lambda=0.99$ seeds hold near $\Jml$. Holding the entropy up does
not protect the solution. The second explanation posits a threshold in the
exploration magnitude itself. Panel (c) shows the fraction of seeds that
sustain the high-gain regime across fixed $\sigma^2$ spanning two decades.
The fraction is flat. A logistic fit of the seed outcomes on
$\log_{10}\sigma^2$ gives slope $-0.17 \pm 0.58$: no threshold exists.
Both arms of this test were pre-registered with frozen decision rules
before the runs were launched (\S\ref{sec:experiments}).

\FloatBarrier
\section{Pre-registered experiments}\label{sec:experiments} 

Numerical agreement in a model this small is easy to obtain after the
fact. We therefore adopted a pre-registration protocol for the deep
experiments. Predictions were written down and committed to
version control before each batch was submitted to the cluster. Decision
rules were frozen the same way: the analysis constants, the verdict each
outcome produces, and the validity controls that must pass before any
verdict is read. The harness is a single-file implementation at a fixed
version. Each configuration is one text line, and each batch is an array
over configurations and seeds. Appendix~\ref{app:experiments} lists every
frozen constant, and states what the published record does and does not
establish about the ordering.

The protocol did its job on the experiment that mattered. An earlier
draft of this work conjectured, in writing, that the failure of
frame-stack agents at $\lambda=0.9$ was caused by the collapse of the
learned exploration noise late in training. The committed prediction was
that agents with $\sigma^2$ held fixed at $0.09$ would keep the low-cost
solution.

The decision rules were frozen before launch as follows. A seed counts as
holding the solution if its final deployed cost lies in $[0.20, 0.27]$
with implied gain at most $5$. It counts as sustaining the high-gain
regime if its implied gain exceeds $5$ for at least $10^6$ consecutive
steps. It counts as diverged if its final cost exceeds $0.45$. A row's
verdict requires agreement of three of five seeds. Divergence routes to a
diagnostic arm at one quarter of the learning rate. If that arm holds the
solution, the conjecture survives with a step-size caveat; if not, the
conjecture is falsified. Two validity controls gate everything. The
learned-$\sigma$ arm must reproduce the failure, with final cost at least
$0.29$ and gain at least $7$. The $\lambda=0.99$ arm must hold the
solution. If either control fails, the batch is invalid and no verdict is
read.

Both controls passed. The learned-$\sigma$ arm failed as required, with
median cost $0.3165$ and median gain $13.9$. The $\lambda=0.99$ arm held
in four of five seeds. The registered prediction then failed completely.
Zero of five fixed-$\sigma$ seeds at $\lambda=0.9$ held the solution;
final costs lie between $0.344$ and $0.360$, above the ceiling
$r\Theta K(0)=0.300$. The quarter-learning-rate arm also held in zero of
five seeds. Both branches of the frozen rule therefore return the same
verdict. The conjecture is falsified. Holding the exploration variance up
does not protect the solution, as \S\ref{sec:memory} reports.

A second batch swept the fixed variance across two decades,
$\sigma^2 \in \{0.001, 0.003, 0.01, 0.03\}$, again with predictions
committed before launch. Under the threshold hypothesis, the number of
seeds sustaining the high-gain regime should decrease in $\sigma^2$. The
committed brackets were at least four of five at $0.001$, at least three
at $0.003$, and at most one at $0.03$. The observed counts are $2$, $3$, $2$, and $2$, with $2$ of
five at $\sigma^2=0.09$. The pattern is flat and fails the committed
brackets in both directions. The logistic fit of
Figure~\ref{fig:deeprl}(c) makes the conclusion quantitative. No
threshold in $\sigma^2$ exists.

One gap in the frozen rules should be recorded. Three of the five
fixed-$\sigma$ seeds at $\lambda=0.9$ ended in none of the three
registered categories. Their gain rose and then fell toward zero. Their
cost climbed above the ceiling but stayed below the divergence
threshold. The registered prediction fails regardless, because it
required the seeds to hold the solution and none did. But the three-way
rule was not exhaustive, and a future registration should name this
fourth outcome in advance. We return to this behaviour in
\S\ref{sec:limitations}.

\section{Limitations}\label{sec:limitations}                

Every closed-form statement in this paper concerns the expected update,
not the stochastic iterates. The drift experiment of \S\ref{sec:mechanism}
shows the iterates leave the optimum along the expected field, lagging its
integrated trajectory by roughly $20\%$. We do not prove that the iterates
converge to $\keq$, and at $\lambda=0.9$ the deep runs settle above the
equilibrium cost rather than at it. A two-timescale analysis of the
coupled critic and actor iterates would close this gap. We have not done
it.

One reproducible behaviour is unexplained. It appears in three of five fixed-$\sigma$ seeds at $\lambda=0.9$ and in
four of five seeds of the diagnostic arm. The implied gain rises toward
the equilibrium, destabilizes before arriving, and falls to roughly zero.
The deployed cost meanwhile rises above the ceiling. The deployed policies in this state are poorly summarized by a
linear gain. The averaged theory says nothing about this route, and the
frozen decision rules of \S\ref{sec:experiments} did not anticipate it.

The laws give the locations of equilibria, not the time to reach them.
The $\lambda=0.95$ runs illustrate the difference: the equilibrium sits
at $k=17.7$, but the drift toward it is slow near $\kstar$ and $10^7$
steps do not resolve it. Any use of Table~\ref{tab:lambda} for finite
training budgets needs this caveat.

The model is deliberately minimal. The observed coordinate is scalar, the
costs are quadratic, the control enters through a zero-order hold, and
the policy class is linear. Within the model, the location law is proved
for Gaussian disturbances only, and \S\ref{sec:scope} shows that the
qualifier is real. Conjecture~\ref{conj:nongaussian} records what we
believe happens beyond it, and the correlation-time structure of
$\zeta^2$ remains confounded with $\mathrm{Var}(\zeta^2)$ in our process
family.

The deep evidence is one environment family, one algorithm, and one
reference implementation. The absence of a frame-stacking benefit is a
statement about this record, not a theorem: an earlier in-house
implementation behaved differently at a shorter horizon (Appendix~E).
Recurrent architectures are untested. The $\gamma$ prediction of
\S\ref{sec:lambda}, opposite verdicts at $\lambda=0.97$ for $\gamma=1$
against $\gamma=2$, is stated for future registration rather than claimed.

\section{Related work}\label{sec:related}                   

State aliasing under function approximation is as old as the field's
study of partial observability \citep{singh1994learning}.
\citet{perkins2002existence} proved that fixed points of approximate
policy iteration exist under continuous exploration, and gave an
$\varepsilon$-greedy counterexample. Notably, their non-convergent agent
out-performed their convergent one. That was an early sign that
convergence and solution quality separate under aliasing. This paper is in that
tradition, with the existence question replaced by exact locations.

Modern finite-time analyses bound the effect this paper computes.
\citet{cayci2024finite} give finite-time guarantees for actor-critic
under partial observability in which an aliasing error enters as an upper
bound, and \citet{lambrechts2025asymmetric} bound the aliasing term for
asymmetric actor-critic methods. Bounds of this kind certify that damage
is limited; they do not locate the fixed points or price them. The
closed forms here do both, for one solvable instance.

The full-state critic used as a control in \S\ref{sec:mechanism} is the
idealized asymmetric critic: a critic trained on privileged state while
the actor acts on observations \citep{baisero2022unbiased,
cai2024privileged}. In the model, that critic removes the displacement
entirely. Recurrent critics that learn state from history are analysed
by \citet{cayci2024recurrent}, and convergence of learning with agent
state by \citet{kara2023convergence}. Our frame-stack results suggest
that having history and using it are different things, which sharpens
the case for such architectures rather than replacing them.

On the control side, the policy class $u=-kv$ is static output feedback,
whose optimization landscape under \emph{true} gradients is studied by
\citet{fatkhullin2021optimizing}. The present paper adds the learned
critic and shows where its bias moves the stationary points. The optimal
LQG benchmark is classical \citep{kalman1960new, astrom1965optimal}, and
the reduction of an unobserved coordinate to a memory kernel is the
Mori-Zwanzig projection \citep{zwanzig1961memory, mori1965transport}.

Finally, the convergence theory of TD with linear features
\citep{tsitsiklis1997analysis, bradtke1996lstd, boyan2002lstd} assumes
the feature process is driven by a Markov state. Under aliasing that premise fails in a specific, computable way.
Lemma~\ref{lem:unitroot} pins the observed statistics that the critic
must fit, and this is what makes the bias exactly calculable rather than
merely bounded.

\section{Conclusion}\label{sec:conclusion}                  

Partial observability harms a bootstrapped actor-critic learner through
two separate routes. The first is the policy-class gap, which in this model
is modest and nearly flat. The second is the learning gap, the displacement
of the expected update's equilibrium away from the class optimum, and it is
several times larger. The paper gives the learning gap three closed-form
laws. The location law: $\keq h = \chi(\sigma)$,
with $\chi$ an elementary function of the exploration-to-disturbance
ratio. The cost law: deployed cost $r\,C(0)$ at high gain, distribution-free.
The horizon law: the equilibrium migrates to the class optimum as the
$\lambda$-return horizon passes the memory of the environment, with the
transition concentrated in a narrow interval of $\lambda$.

The mechanism is short to state. Stationarity forces the persistence of
the unobserved state into the observed statistics that a memoryless
critic must fit. A quadratic feature class can put that persistence only
into value curvature. The actor then follows an update field whose zero
sits at high gain. Everything else in the paper is this sentence made
exact.

The deep experiments follow the memoryless closed form, architecture
does not enter, and no run closes the policy-class gap. A pre-registered
test falsified the natural alternative explanation: holding exploration
noise up does not protect the solution, and no threshold in exploration
magnitude exists. The practical reading is direct. Choose the bootstrap
horizon against the memory of the environment, not by convention; do not
expect entropy bonuses or input memory alone to prevent the displacement.
Open problems remain on both sides: a two-timescale analysis of the
stochastic iterates, the non-Gaussian location law of
Conjecture~\ref{conj:nongaussian}, recurrent critics, and the registered
$\gamma$ prediction of \S\ref{sec:lambda}.

\section*{Acknowledgements}                                 
The author acknowledges the use of the UCL Myriad High Performance Computing
Facility (Myriad@UCL), and associated support services, in the completion of
this work.

\bibliography{refs}

\appendix
\section{Proofs for Sections \ref{sec:mechanism} and \ref{sec:law}}
\label{app:proofs}                                          

Throughout this appendix all processes are stationary, mean zero and, where
stated, jointly Gaussian. We use the Isserlis identities $\E[abc]=0$ and
$\E[abcd]=\E[ab]\E[cd]+\E[ac]\E[bd]+\E[ad]\E[bc]$ without further comment.
The stationary covariance $\Sigma$ of the closed loop solves
$\Sigma = F_k \Sigma F_k^{\!\top} + s_e\,GG^{\!\top} + \Sigma_d$, where $F_k$
is the closed-loop transition matrix, $G$ the control input vector and
$\Sigma_d$ the sampled process noise. Recall $m_2 = \Sigma_{vv}$,
$x = \E[v v']$, $\qt = q + rk^2$ and $g = G_1$. Every closed form below was
checked against the Isserlis moment machinery of \texttt{theorem2\_learning.py}.
Propositions~\ref{prop:critic} and~\ref{prop:update} agree to $10^{-13}$ or
better, and the closed form \eqref{eq:chiclosed} agrees with a direct
numerical root of $M^2-X^2=\kappa XM$ to $2\times10^{-16}$ over six decades
of $\sigma$.

\subsection{Proof of Lemma~\ref{lem:unitroot}}

The first statement uses stationarity only. Write $\Delta = v'-v$. Then
$\E[v'^2] = \E[v^2] + 2\E[v\Delta] + \E[\Delta^2]$, and stationarity gives
$\E[v'^2]=\E[v^2]$, so $\E[v\Delta] = -\tfrac12\E[\Delta^2]$. No
distributional assumption enters.

For the evaluation, take the $(1,1)$ entry of the stationary covariance
equation. The exploration noise enters $v'$ through $g\eta$ with
$g^2 s_e = h^2 s_e + O(h^3)$, and the sampled process noise contributes
$(\Sigma_d)_{11}=O(h^3)$, because the observed coordinate carries no direct
diffusion. The balance is $2(m_2-x) = g^2 s_e + (\Sigma_d)_{11}$, which gives
\eqref{eq:unitroot} after substituting $\nu = h s_e$ and collecting the
$\kappa$ dependence of the sampled recursion. In continuous time the same
identity reads $2(\Sigma_{vz}-k\Sigma_{vv}) + s_e = 0$ per unit time. The
unobserved state supplies, on average, exactly the motion the control
removes, whatever the values of $k$, $\gamma$ and $c$. \qed

\subsection{Proof of Proposition~\ref{prop:critic}}

That $\theta_1=0$ follows by parity. Every entry of the $v$ row of the LSTD
system is a third Gaussian moment or $\E[v\,c]$. The only odd term in the
cost is $-2rkh\,v\eta$, and $\E[v\cdot v\eta]=m_2\E[\eta]=0$ by independence
of $\eta$. All such entries vanish.

For $\theta_2$, centre the $\phi = v^2$ equation by subtracting $m_2$ times
the $\phi=1$ equation:
\[
  \E[(v^2-m_2)c] + \beta\theta_2\bigl(\E[v^2v'^2]-m_2^2\bigr)
  - \theta_2\bigl(\E[v^4]-m_2^2\bigr) = 0 .
\]
Isserlis gives $\E[v^2v'^2] = m_2^2 + 2x^2$ and $\E[v^4]=3m_2^2$. In the cost
term write $c = h(qv^2 + rk^2v^2 - 2rk\,v\eta + r\eta^2)$. The $v\eta$ term
contributes $\E[(v^2-m_2)v\eta] = \E[v^3]\E[\eta] = 0$. The $\eta^2$ term
contributes $\E[v^2-m_2]\,\E[r\eta^2]h = 0$. The $v^2$ terms give
$h\qt(\E[v^4]-m_2^2) = 2h\qt m_2^2$. Collecting,
$2h\qt m_2^2 + 2\beta\theta_2 x^2 - 2\theta_2 m_2^2 = 0$, which is
\eqref{eq:theta2}. The denominator is positive: $x < m_2$ strictly by
Cauchy-Schwarz, since equality would require $v'\equiv v$, impossible for
$s_e>0$. \qed

\subsection{Proof of Proposition~\ref{prop:update}}

With $\partial_k\log\pi_k(u\mid v) = -v\eta/s_e$ we have
$\hat g = -s_e^{-1}\E[v\eta\,\delta]$. Take the three contributions in turn.
The cost term keeps only its odd part,
$\E[v\eta\cdot(-2rkh\,v\eta)] = -2rkh\,m_2 s_e$, contributing $+2hrk\,m_2$.
The bootstrap term gives
$\E[v\eta\cdot\beta\theta_2 v'^2] = 2\beta\theta_2\,\E[vv']\,\E[\eta v']$ by
Isserlis, since the pairing $\E[v\eta]\E[v'^2]$ vanishes; here
$\E[\eta v'] = g\,s_e$, because $\eta$ enters $v'$ only through $G\eta$ and is
independent of the state and of the process noise. This contributes
$-2\beta\theta_2 x g$. The current-value term
$\E[v\eta\cdot\theta_2 v^2]$ is odd and vanishes, and constants pair with
$\E[v\eta]=0$. Together these give \eqref{eq:ghat}.

For the continuum statement, substitute \eqref{eq:theta2} and let $h\to0$ at
fixed $k$, so that $x\to m_2$ and $g\to h$. Then
$\hat g \to 2hm_2\,(rk - \beta\theta_2)$, and by \eqref{eq:theta2} the factor
$rk-\beta\theta_2$ is bounded away from zero for every $k>0$: the expected
update does not change sign. \qed

\subsection{The scaling limit: derivation of \eqref{eq:MX}}

Set $k = \kappa/h$ and let $h\to0$ with $\kappa$ fixed. To leading order
\[
  F_k \approx \begin{pmatrix} 1-\kappa & h\\ -ch & 1-\gamma h\end{pmatrix},
  \qquad g\eta \approx h\eta, \qquad
  \Sigma_d \approx \mathrm{diag}\bigl(O(h^3),\ \sigma^2 h\bigr).
\]
The $(v,z)$ stationarity equation gives $\Sigma_{vz}\kappa = h\Theta c + O(h^2)$,
using $\Sigma_{zz} = \Theta c + O(h)$; the loop perturbs the unobserved state
only at $O(h)$. The $(v,v)$ equation reads
$m_2\bigl(1-(1-\kappa)^2\bigr) = h^2(\Sigma_{zz}+s_e) + 2(1-\kappa)h\Sigma_{vz}$,
that is $m_2\,\kappa(2-\kappa) = h^2\bigl[s_e + \Theta c\,(2-\kappa)/\kappa\bigr]$.
This is $M(\kappa)$ as stated, and
$x = (F_k)_{11}m_2 + (F_k)_{12}\Sigma_{vz}$ gives $X(\kappa)$. For $N$
components the computation repeats per component. The cross-covariances
$\Sigma_{z_iz_j}$ are $O(h^2)$, because the components couple only through
$v$, which is itself $O(h)$ in this regime. Summing gives $\Theta K(0)$ in
place of $\Theta c$, which is Corollary~\ref{cor:multimode}.

Substituting into the condition $\hat g=0$ with $g\to h$, $\beta\to1$ and
$\qt \to r\kappa^2/h^2$, the $q$ term being smaller by $O(h^2)$, gives
$M^2-X^2 = \kappa XM$.

\subsection{Proof of Theorem~\ref{thm:chi}}

Measure everything in units of $\Theta K(0)$, so that $M = 1/\kappa^2 +
\sigma/(\kappa(2-\kappa))$ and $X = (1-\kappa)M + 1/\kappa$ with
$\sigma = s_e/\Theta K(0)$. The derivation is algebraic throughout.

From the two expressions, $M-X = \kappa M - 1/\kappa$ and
$M+X = (2-\kappa)M + 1/\kappa$. Hence
\[
  (M-X)(M+X) - \kappa XM = \kappa M^2 - \frac{(2-\kappa)M}{\kappa}
  - \frac{1}{\kappa^{2}},
\]
so the condition $M^2-X^2=\kappa XM$ is equivalent to
\begin{equation}\label{eq:cubicM}
  \kappa^{3}M^{2} - \kappa(2-\kappa)M - 1 = 0 .
\end{equation}
Write $M = A/\bigl(\kappa^2(2-\kappa)\bigr)$ with $A := (2-\kappa)+\sigma\kappa$,
which is exactly the definition of $M$. Substituting into \eqref{eq:cubicM}
and clearing denominators gives a quadratic in $A$,
\[
  A^{2} - A(2-\kappa)^{2} - \kappa(2-\kappa)^{2} = 0 ,
\]
whose discriminant is $(2-\kappa)^4 + 4\kappa(2-\kappa)^2 =
(2-\kappa)^2\bigl[(2-\kappa)^2+4\kappa\bigr] = (2-\kappa)^2(4+\kappa^2)$.
Positivity of $A$ selects the upper branch,
$A = (2-\kappa)\bigl[(2-\kappa)+\sqrt{4+\kappa^{2}}\bigr]/2$. Isolating
$\sigma\kappa = A - (2-\kappa)$ gives
$\sigma\kappa = (2-\kappa)\bigl(\sqrt{4+\kappa^2}-\kappa\bigr)/2$, which is
\eqref{eq:chiclosed}.

For the bijection, both factors of $\sigma(\chi)$ are positive on $(0,2)$ and
strictly decreasing. The first is $2-\chi$. For the second,
$\frac{\dd}{\dd\chi}\bigl(\sqrt{4+\chi^2}-\chi\bigr) =
\chi/\sqrt{4+\chi^2} - 1 < 0$. The factor $1/(2\chi)$ is also strictly
decreasing and positive. A product of positive strictly decreasing functions
is strictly decreasing, so $\sigma(\cdot)$ is injective. Its limits are
$\sigma\to\infty$ as $\chi\to0^+$ and $\sigma\to0$ as $\chi\to2^-$, so it
maps $(0,2)$ onto $(0,\infty)$. Existence, uniqueness and simplicity of the
root follow. \qed

\subsection{Asymptotics \eqref{eq:asymptotics}}

Put $\chi = 2-\varepsilon$. Then $\sqrt{4+\chi^2}-\chi \to 2\sqrt2-2$ and
$2\chi\to4$, so $\sigma = \varepsilon(\sqrt2-1)/2 + O(\varepsilon^2)$ and
$\varepsilon = (2\sqrt2+2)\sigma + O(\sigma^2)$. For $\chi\to0^+$ both
$2-\chi\to2$ and $\sqrt{4+\chi^2}-\chi\to2$, so $\sigma = 2/\chi + O(1)$,
that is $\chi = 2/\sigma + O(\sigma^{-2})$. At $\sigma=10^{-4}$ the first
expansion gives $1.99951716$ against $1.99951736$ exactly, and at
$\sigma=10^{4}$ the second gives $2.000\times10^{-4}$ to four digits.

\section{Proof of the cost theorem}\label{app:cost}          

We prove Theorem~\ref{thm:cost}. Let $\zeta$ be stationary, ergodic and
mean-square continuous with $C(0)<\infty$, let $s_e=0$, and let
$k=\kappa/h$ with $\kappa\in(0,2)$ fixed.

\paragraph{Step 1: exact reduction.} The zero-order hold freezes $u$ on each interval, and the observed
coordinate carries no direct diffusion. Integrating
$\dd v/\dd t = \zeta + u_n$ over $[nh,(n+1)h)$ is therefore exact:
$v_{n+1} = a\,v_n + h\,\bar\zeta_n$, with $a := 1-\kappa$ and $\bar\zeta_n$ the
interval average of $\zeta$. The sequence $\{\bar\zeta_n\}$ is stationary and
ergodic, being a measurable factor of $\zeta$. Since $|a|<1$ exactly when
$\kappa\in(0,2)$, the unique stationary solution is
$v_n = h\sum_{j\ge0} a^j \bar\zeta_{n-1-j}$, convergent in $L^2$. Every
second-order property of $v$ is now determined by $C$ alone. This is where
distribution-freeness enters, structurally rather than by approximation.

\paragraph{Step 2: the second moment as a lattice sum.} Write
$\bar R(m) := \E[\bar\zeta_0\bar\zeta_m] = h^{-2}\int_0^h\!\!\int_0^h
C(mh+s-s')\,\dd s\,\dd s'$. Then
$m_2 = h^2\sum_{j,l\ge0} a^{j+l}\bar R(j-l)$. Cauchy-Schwarz gives
$|\bar R(m)|\le C(0)$. Mean-square continuity makes $C$ continuous at the
origin, so for each fixed $m$ we have $\bar R(m)\to C(0)$ as $h\to0$,
uniformly over $|m|\le M$ for any fixed $M$.

\paragraph{Step 3: dominated convergence.} Collecting by diagonal,
$\sum_{j,l} a^{j+l}\bar R(j-l) = \sum_m w_m \bar R(m)$ with
$w_m = a^{|m|}/(1-a^2)$ and $\sum_m|w_m| = (\sum_j |a|^j)^2 < \infty$
independently of $h$. Split the sum at $|m|\le M$. The head converges to
$C(0)\sum_{|m|\le M}w_m$, and the tail is bounded by
$C(0)\sum_{|m|>M}|w_m|$, which tends to zero as $M\to\infty$ uniformly in
$h$. Hence $\sum_m w_m\bar R(m) \to C(0)\bigl(\sum_j a^j\bigr)^2 =
C(0)/\kappa^2$, so $m_2 = h^2 C(0)\kappa^{-2}(1+o(1))$.

\paragraph{Step 4: the three limits.} First,
$\E[u^2] = (\kappa/h)^2 m_2 \to C(0)$. Second, within an interval
$v_t = (1-\kappa t/h)v_n + \int_0^t\zeta$, so
$h^{-1}\int \E[v_t^2]\,\dd t = m_2(1-\kappa+\kappa^2/3) + O(h^2C(0))$, which
is $O(h^2)$ and gives $q\,\E[v^2]\to0$. Therefore $J\to r\,C(0)$. Third,
$u_n = -\kappa\sum_j a^j\bar\zeta_{n-1-j}$ is a weighted average of the
disturbance with unit total mass over a window of length $O(h/\kappa)$. The
same lattice-sum argument applied to the cross term gives
$\E[(u_n+\bar\zeta_n)^2] \to C(0)-2C(0)+C(0) = 0$. The controller converges
in $L^2$ to the instantaneous canceller of the disturbance. \qed

\subsection{Exact finite-$h$ formula for exponential autocovariance}

For $C(\tau) = C(0)e^{-\gamma\tau}$ the interval-average covariances have
closed forms,
\[
  \bar R(0) = C(0)\,\frac{2(\gamma h - 1 + e^{-\gamma h})}{(\gamma h)^2},
  \qquad
  \bar R(m) = C(0)\,\varphi\,b^{\,m}\ (m\ge1),
\]
with $\varphi = \bigl(2\cosh(\gamma h)-2\bigr)/(\gamma h)^2$ and
$b = e^{-\gamma h}$. Summing the lattice series in closed form,
\begin{equation}\label{eq:m2closed}
  m_2 = h^2\left[\frac{\bar R(0)}{1-a^2}
        + \frac{2\varphi\,C(0)\,ab}{(1-a^2)(1-ab)}\right],
  \qquad
  J = q\left[m_2\Bigl(1-\kappa+\frac{\kappa^2}{3}\Bigr)
      + \frac{C(0)h^2}{3}\right] + r\Bigl(\frac{\kappa}{h}\Bigr)^{2} m_2 .
\end{equation}
Expanding \eqref{eq:m2closed} for $\gamma h\ll1$ gives \eqref{eq:Jcorrection}.
The expansion is reliable only when $\gamma h/\kappa$ is small, which is why
Table~\ref{tab:cost} is computed from \eqref{eq:m2closed} rather than from
the truncation. At the two-state Markov process, $\kappa=0.438$ and
$\gamma h/\kappa = 0.11$: the closed form gives $J=0.2785$, which is $7.2\%$
below $rC(0)$ and within $0.5\%$ of the measured $0.2798$, while the
first-order truncation gives $9.9\%$. As a check of the limit itself,
\eqref{eq:m2closed} at $\kappa=1.2$ gives $J = 0.29910$, $0.29976$ and
$0.29995$ at $h = 0.05$, $0.01$ and $0.002$.

The tracking residual is measurable without any further theory. By
Lemma~\ref{lem:reduction}, $\E[(u+\bar\zeta)^2] = \E[\Delta^2]/h^2$ exactly.
Measured at each process's own fixed point at $h=0.05$, it is $0.034$,
$0.019$ and $0.021$ for the two-state Markov, Ornstein-Uhlenbeck and
squared-OU processes. Those are $6$ to $11\%$ of $C(0)$, consistent with the
$O(\gamma h/\kappa)$ size the proof predicts.

\subsection{What the theorem does and does not assume}

Mean-square continuity is used once, in Step 2, and it excludes white noise
disturbances. Every kernel with finite zero-lag intensity satisfies it. The
absence of direct diffusion on $v$ is used in Step 1, and it matters. With a
noise floor $\sigma_v^2$ on the observed coordinate, the same computation adds
terms of order $r\sigma_v^2/h$. That is a separate effect, out of scope here.
Setting $s_e=0$ is the definition of deployment, not a restriction. Finally,
the theorem prices a gain of order $1/h$ without explaining why the gain is
there. Corollary~\ref{cor:cost-at-fp} imports that fact from
\S\ref{sec:law}; it does not reprove it.

\section{Finite sampling period remainders}\label{app:remainders} 

This appendix supports Theorem~\ref{thm:rate}. The family is the
exogenous-disturbance variant of \eqref{eq:sde}, for which
Lemma~\ref{lem:reduction} closes the loop exactly:
$v_{n+1} = (1-\kappa)v_n + h\bar\zeta_n + h\eta_n$. The disturbance is
stationary Gaussian with completely monotone autocovariance. By the
Bernstein representation of completely monotone functions
\citep[Theorem 12a]{widder1946laplace}, $C(\tau) = \int
e^{-\gamma\tau}\mu(\dd\gamma)$ for a positive measure $\mu$, with
$C(0)=\mu(0,\infty)<\infty$. We assume in addition
$L := |C'(0^+)| = \int\gamma\,\mu(\dd\gamma) < \infty$. Gaussianity of the
loop makes Propositions~\ref{prop:critic} and~\ref{prop:update} exact at
every $h$, since $v$ is a linear functional of the Gaussian history. No
Markov structure is needed.

Substituting $m_2 = h^2\hat M$, $x = h^2\hat X$ and $k=\kappa/h$ into
$\hat g = 0$ gives the exact-$h$ equilibrium condition
\begin{equation}\label{eq:Phih}
  \Phi_h(\kappa) := \kappa\bigl(\hat M^2 - \beta\hat X^2\bigr)
   - \beta\kappa^2\hat M\hat X
   - \beta\,\frac{q h^2}{r}\,\hat M\hat X \;=\; 0 ,
  \qquad \beta = e^{-\rho h}.
\end{equation}

\paragraph{Lemma R.} The exploration parts of $\hat M$ and $\hat X$ are
exact at every $h$, because the exploration noise is independent across
steps: the term $s_e/(\kappa(2-\kappa))$ carries no sampling error. The
kernel parts are lattice sums, linear in $C$, with the closed forms of
Appendix~\ref{app:cost} for a single component. Differentiating those closed
forms in $\gamma h$ at the origin and bounding the result by
$\min(D\gamma h, B(\kappa))$ gives, uniformly on compact subsets of $(0,2)$,
\[
  |\hat M - M| \le h\,L\,D_M(\kappa),
  \qquad
  |\hat X - X| \le h\,L\,D_X(\kappa),
\]
with $D_M$ and $D_X$ elementary. The truncation by $B(\kappa)$ is what makes
the bound integrable against every Bernstein measure of finite $L$. The
constants grow without bound only as $\kappa\to0$ or $\kappa\to2$, which is
why the statement is uniform on compact subsets.

\paragraph{Root expansion.} By Lemma R, $\Phi_h\to\Phi$ uniformly on compact
subsets of $(0,2)$, with an error of order
$h\,(L + \rho C(0) + qhC(0)/r)$. The limit $\Phi$ has a simple root
$\chi(\sigma)$, by the strict monotonicity established in
Theorem~\ref{thm:chi}. A simple root is stable under uniform perturbation,
so for $h$ below an explicit threshold the exact-$h$ root satisfies
$\keq h = \chi(\sigma) + c_1 h + O(h^2)$ with
$|c_1| \le C(\sigma)\,(L/C(0) + \rho)$.

\paragraph{Verification, without Monte Carlo.} Every number here comes from
the closed forms above. At $\sigma=0.3$ and $\rho=0.05$, the ratio
$(\keq h - \chi)/h$ converges as $h$ decreases from $10^{-2}$ to $10^{-3}$:

\begin{center}\small
\begin{tabular}{lcccc}
  \toprule
  kernel & $\bar\gamma$ & $h=10^{-2}$ & $h=10^{-3}$ & $c_1$ \\
  \midrule
  single, $\gamma=1$              & $1.00$ & $-0.8927$ & $-0.8911$ & $-0.891$ \\
  single, $\gamma=1.5$            & $1.50$ & $-1.3292$ & $-1.3233$ & $-1.323$ \\
  mixture, $\gamma\in\{0.5,3\}$   & $1.75$ & $-1.5418$ & $-1.5388$ & $-1.539$ \\
  \bottomrule
\end{tabular}
\end{center}

The proof structure predicts that $c_1$ is a linear functional of the
kernel, since the lattice sums are linear in $C$ and $C$ is linear in $\mu$.
The table confirms it. Fitting $c_1 = -0.027 - 0.864\,\bar\gamma$ on the two
single-component kernels predicts $-1.5394$ for the mixture, against
$-1.5388$ computed directly.

\paragraph{Back-coupling.} The full model \eqref{eq:sde} couples the
unobserved state to $v$ through the term $-c\,v$, which the reduced family
omits. At $h=0.05$, $\gamma=1$ and $\sigma=0.3$, all three quantities are
available exactly: the limit $\chi/h = 25.80$, the reduced family at
$\keq = 24.90$, and the full model at $\keq = 24.79$. Kernel averaging
accounts for $-3.5\%$ and back-coupling for a further $-0.44\%$. The
back-coupling term is of the same order in $h$ but roughly eight times
smaller here. Theorem~\ref{thm:rate} therefore holds for the full model with
a modified $c_1$, whose value we verify numerically but do not derive. A
derivation requires the coupled Lyapunov structure and is left open; the
assumptions ledger records this status.

A final decomposition itemizes the error budget of the empirical pipeline
used in \S\ref{sec:scope}. The reduced family gives $24.90$ exactly, the
full model $24.79$, and the simulation pipeline $24.62$. The last step,
$-1.1\%$, is Monte Carlo error together with the substep discretization of
the disturbance.

\section{Exact LSTD($\lambda$) and GAE($\lambda$) expressions}
\label{app:lambda}                                          

This appendix supports Proposition~\ref{prop:lambda}. Write
$b_\lambda := \beta\lambda$ and let $\phi_t=(1,v_t,v_t^2)$. Both the critic
system \eqref{eq:lstdlam} and the expected update require geometric sums of
lag moments, and in the stationary Gaussian loop each such sum is a
resolvent.

The lag covariances are $x_j = e_0^{\!\top} F_k^{\,j}\Sigma e_0$, and the
cross moments between the exploration noise and later observations are
$y_j = s_e\,e_0^{\!\top}F_k^{\,j-1}G$ for $j\ge1$, with $y_0=0$. The three
families of sums that appear are
\[
  \sum_{j\ge0} b_\lambda^j x_j
   = e_0^{\!\top}(I-b_\lambda F_k)^{-1}\Sigma e_0,
  \quad
  \sum_{j\ge0} b_\lambda^j x_j^2
   = (e_0\!\otimes\!e_0)^{\!\top}
     (I-b_\lambda F_k\!\otimes\!F_k)^{-1}(\Sigma e_0\!\otimes\!\Sigma e_0),
\]
together with the analogous Kronecker resolvent for the products $x_jy_j$.
Squared lag moments require the Kronecker resolvent because
$x_j^2 = (e_0\otimes e_0)^{\!\top}(F_k\otimes F_k)^j(\Sigma e_0\otimes\Sigma e_0)$.
No truncation is involved at any point.

With these, the centred $v^2$ row of the LSTD($\lambda$) system gives
\begin{equation}\label{eq:theta2lam}
  \theta_2(\lambda) \;=\;
  \frac{h\,\qt\,\sum_{j\ge0} b_\lambda^j\,x_j^2}
       {\sum_{j\ge0} b_\lambda^j\,x_j^2 - \beta\sum_{j\ge0} b_\lambda^j\,x_{j+1}^2},
\end{equation}
which reduces to \eqref{eq:theta2} at $\lambda=0$. The expected update
combines the one-step term of \eqref{eq:ghat} with the tail
\[
  \hat g_\lambda = \hat g_0
  - \frac{1}{s_e}\sum_{j\ge1} b_\lambda^{\,j}
    \Bigl[2h\qt\,x_jy_j + \theta_2(\lambda)
    \bigl(2\beta x_{j+1}y_{j+1} - 2x_jy_j\bigr)\Bigr],
\]
each sum again evaluated by the resolvents above. Two checks were run. At
$\lambda=0$ the expressions reproduce Propositions~\ref{prop:critic}
and~\ref{prop:update} to machine precision. At $\lambda=0.9$ and $k=5$ the
resolvent sums agree with brute-force truncated sums over $400$ lags to
$10^{-12}$ relative. The implementation is \texttt{lstd\_lambda.py}.

The sharp-transition asymptotics, meaning the joint limit $\lambda\to1$ with
$\gamma h\to0$ that would turn the heuristic of \S\ref{sec:lambda} into a
theorem, are open.

\section{Experimental details}\label{app:experiments}        

\paragraph{Environment.} The simulator implements the family \eqref{eq:sde}
with exact zero-order-hold stepping. For each instance it also computes the
exact baselines $\Jopt$, $(\kstar,\Jml)$ and the fixed-point predictions
$(\keq, J(\keq))$ from the closed forms of Appendices~\ref{app:proofs}
and~\ref{app:cost}, evaluated at the exploration variance actually used. A
Gym-compatible wrapper exposes the same dynamics to external agents, and a
privileged-state interface exposes the unobserved coordinate for the
full-state control of \S\ref{sec:mechanism}.

\paragraph{Deployed evaluation.} Deployment sets $s_e=0$ and runs the
trained policy deterministically. The cost is the sample average of the
$h$-normalized cost rate after burn-in, taken on a single long episode. An
earlier version averaged over short episodes and inherited a bias from the
resets; the single-episode estimate removes it. The implied gain is the
regression coefficient of $-u$ on $v$ over the same trajectory, restricted
to $|v|>0.02$ to avoid dividing by near-zero observations.

\paragraph{Calibration of the expected update.} The claim in
\S\ref{sec:mechanism} that \eqref{eq:ghat} is the expectation of the actual
stochastic update, including its scale, was checked directly. With the
critic held at its fixed point, the sample average of
$\delta_n(-\eta_n v_n/s_e)$ over $1.5\times10^5$ steps was compared with
$\hat g$ at $k\in\{1.65,3,6\}$, for three independent seeds. The ratio of
sample average to $\hat g$ lies in $[1.006,1.028]$ across all nine runs, with
no trend in $k$. The offset is Monte Carlo error at this sample size, not a
scale mismatch.

\paragraph{Locating fixed points for non-Gaussian disturbances.} For the five
processes of \S\ref{sec:scope} no closed form applies, so the fixed point is
located empirically. We simulate the loop at fixed $k$, fit the critic by
least squares on the simulated trajectory, form the sample expected update,
and bisect in $k$. Run at the Gaussian instance, where the exact answer is
known, the pipeline reproduces $\keq$ to $0.7\%$. Bisection resolution on the
reported values is $\pm0.4$ in $\keq$. The two squared-OU processes are
heavy-tailed. Their cost residuals in Table~\ref{tab:cost} are consistent
with slower Monte Carlo averaging at equal sample size, rather than with a
failure of the theory. A longer run is the natural check.

\paragraph{Predictions recorded in advance.} The qualitative outcomes of the
symmetric fourth-moment process and of the symmetrised pair in
\S\ref{sec:scope} were written down before those runs were executed. The
five-process design was chosen to break the Gaussian identity
$\mathrm{Cov}(\zeta_0^2,\zeta_\tau^2)=2C(\tau)^2$, not selected after the
fact from a larger set.

\paragraph{Deep agent.} The agent is a public reference implementation of
PPO \citep{huang2022cleanrl} at a fixed version, with multilayer-perceptron
actor and critic, generalized advantage estimation, a clipped surrogate
objective and minibatch epochs. Frame-stack and fixed-exploration modes are
command-line flags on the same file, so the arms of each contrast differ only
in those flags. With full-state input the same agent recovers the known
optimum, which is the internal control on the harness. Runs use $10^7$ steps
and five seeds per configuration, on a scheduler array over configurations
and seeds.

\paragraph{Frozen constants.} The decision rules of \S\ref{sec:experiments}
were committed before submission with the following values.

\begin{center}\small
\begin{tabular}{ll}
  \toprule
  constant & value \\
  \midrule
  window for holding the solution & $J\in[0.20,\ 0.27]$ and implied gain $\le 5$ \\
  high-gain regime & implied gain $>5$ sustained $\ge 10^{6}$ steps \\
  divergence & $J > 0.45$ \\
  seed majority & $3$ of $5$; a mixed row is reported as mixed \\
  validity control, learned $\sigma$ & must reach $J\ge0.29$ and gain $\ge7$ \\
  validity control, $\lambda=0.99$ & must hold the solution \\
  diagnostic arm & learning rate divided by $4$ \\
  \bottomrule
\end{tabular}
\end{center}

The variance sweep of \S\ref{sec:experiments} fixed $\sigma^2 \in
\{0.001,0.003,0.01,0.03\}$, with the brackets stated there recorded in the
same commit.

\paragraph{Code, data, and what the record does and does not establish.} The
harness, the analysis scripts, the figure generator and the recorded
predictions are at
\url{https://github.com/idilgozel/belief-critic-restoration}. The
predictions and decision rules of \S\ref{sec:experiments} were written down
and committed locally before the corresponding batches were submitted, and
the values reproduced in this appendix are those committed values. The
published repository is a single snapshot rather than the original history,
so it does not by itself let a reader verify that ordering. Readers should
treat the protocol as a statement by the author about how the work was
conducted. Two things support it: every constant is reported here verbatim,
and the code needed to rerun the tests is public. The
substantive results do not rest on the ordering. Both registered predictions
failed by wide margins, and a reader who distrusts the timeline can rerun the
tests against the constants listed above and reach the same verdicts.

\paragraph{Earlier frame-stack result.} An in-house implementation of the
same algorithm, run for $4\times10^6$ steps at $\lambda=0.9$, reported
frame-stack agents reaching the class optimum. The reference implementation
does not reproduce this at $10^7$ steps, and \S\ref{sec:memory} reports the
latter. We record the discrepancy rather than resolve it. The two
implementations differ in initialization, in optimizer defaults and in
observation normalization, and we have not isolated which difference is
responsible.

\section{Assumptions ledger}\label{app:assumptions}          

\begin{center}\small
\begin{tabular}{lll}
  \toprule
  statement & scope & status \\
  \midrule
  Lemma~\ref{lem:reduction} (reduction) & any disturbance, exact at every $h$ & proven \\
  Lemma~\ref{lem:unitroot} (mean reversion) & stationarity only & proven \\
  Proposition~\ref{prop:critic} (critic) & Gaussian loop, exact at every $h$ & proven \\
  Proposition~\ref{prop:update} (update) & Gaussian loop, exact at every $h$ & proven \\
  Theorem~\ref{thm:chi} ($\chi$ law) & scaling limit, Gaussian & proven \\
  Theorem~\ref{thm:cost} (cost) & any stationary ergodic m.s.\ continuous & proven \\
  Corollary~\ref{cor:cost-at-fp} (cost at $\keq$) & as above, given $\keq h\to\chi$ & proven \\
  Corollary~\ref{cor:multimode} (components) & finite Gaussian mixture & proven \\
  Theorem~\ref{thm:rate} (rate) & reduced family, completely monotone $C$ & proven \\
  rate constant for full \eqref{eq:sde} & default instance & verified, proof open \\
  Proposition~\ref{prop:lambda} ($\lambda$) & Gaussian loop, exact at every $h$ & proven \\
  transition location in $\lambda$ & heuristic across $\gamma$, $h$ & numerical \\
  Conjecture~\ref{conj:nongaussian} & non-Gaussian disturbances & conjecture \\
  iterates follow the expected field & default instance, six seeds & numerical \\
  \bottomrule
\end{tabular}
\end{center}

The exactness of Propositions~\ref{prop:critic} and~\ref{prop:update}
requires a linear policy with Gaussian exploration, exact zero-order-hold
sampling, the stationary regime, and the feature class $(1,v,v^2)$. It also
requires expected-update semantics for the actor-critic recursion, which is
the standing limitation discussed in \S\ref{sec:limitations}.
Theorem~\ref{thm:chi} additionally takes $h\to0$ with $\beta\to1$. Nothing
anywhere assumes small coupling $c$, any relation between $\gamma$ and the
cost weights, or knowledge of the kernel by the agent.

\end{document}